\documentclass[conference]{IEEEtran}
\IEEEoverridecommandlockouts
\usepackage[T1]{fontenc}
\usepackage[utf8]{inputenc}

\usepackage{graphics}
\usepackage{mathrsfs}
\usepackage{amsmath}
\usepackage{amssymb}
\usepackage{gensymb}
\usepackage{float,bm}
\usepackage{wasysym}
\usepackage{mathtools}
\usepackage{graphicx}
\usepackage[dvipsnames]{xcolor}
\usepackage{colortbl}
\usepackage{multirow}
\usepackage{url}
\usepackage{booktabs}
\usepackage[export]{adjustbox}

\usepackage[caption=false,font=footnotesize]{subfig}

\usepackage{perl_acronyms}
\usepackage{perl_misc}

\usepackage{cite}
\newtheorem{thm}{Theorem}
\newtheorem{lem}[thm]{Lemma}
\newtheorem{assum}[thm]{Assumption}
\newtheorem{defn}[thm]{Definition}
\newenvironment{pf}{\begin{IEEEproof}}{\end{IEEEproof}}
\providecommand{\qed}{}

\definecolor{algreen}{RGB}{0, 180, 0}
\definecolor{blazepurple}{RGB}{150, 0, 150}
\definecolor{phlamegold}{RGB}{218, 165, 32}
\definecolor{aligatororange}{RGB}{255, 140, 0}
\definecolor{papercream}{RGB}{250,247,230}

\definecolor{mycyan}{RGB}{24,248,248}
\definecolor{myblue}{RGB}{8,8,200}

\definecolor{init_traj_blue}{HTML}{5577FE}
\definecolor{final_traj_green}{HTML}{1DB032}
\definecolor{init_pose_turq}{HTML}{3FC8D1}
\definecolor{final_pose_gold}{HTML}{D8A621}
\definecolor{obs_red}{HTML}{C80000}
\definecolor{red_initial_guess}{HTML}{D55347}

\providecommand{\methodname}{\text{PHLAME}}
\providecommand{\methodnamenew}{\text{MAGMA}}

\newcommand{\revision}[2]{#2}

\newcommand{\stkout}[1]{{\color{Orange}\ifmmode\text{\sout{\ensuremath{#1}}}\else\sout{#1}\fi}}

\providecommand{\X}{x}

\providecommand{\A}{\mathcal{A}}

\providecommand{\FF}{\mathcal{F}}

\providecommand{\R}{\ensuremath \mathbb{R}}
\providecommand{\Ri}{\ensuremath \bar{\mathbb{R}}}

\providecommand{\uu}{u}
\providecommand{\x}{\mathtt{x}}
\providecommand{\xd}{\dot{\mathtt{x}} }

\providecommand{\Fdx}{F_d(\x(t))}
\providecommand{\Fx}{F(\x(t))}

\providecommand{\q}{q}
\providecommand{\qd}{\dot{q}}

\providecommand{\Minv}{M^{-1}}

\providecommand{\Act}{\mathcal{A}}
\providecommand{\dxds}{\frac{\partial \X}{\partial s}}

\providecommand{\ddt}{\frac{d}{dt}}

\providecommand{\dldxs}{\frac{\partial L}{\partial \X_s}}
\providecommand{\dldxds}{\frac{\partial L}{\partial \dot{\X}_s}}

\providecommand{\dlpdxs}{\frac{\partial L'}{\partial \X_s}}
\providecommand{\dlpdxds}{\frac{\partial L'}{\partial \dot{\X}_s}}

\providecommand{\ccons}{c_{\text{cons}}}

\providecommand{\Xinit}{\X_{init}}

\providecommand{\Ufb}{u_{fb}}

\begin{document}
\title{MAGMA: Multiplier-Augmented Geometric Motion Algorithm\thanks{This work was supported by AFOSR MURI FA9550-23-1-0400.}}
\author{\IEEEauthorblockN{César E. Ramos Chuquiure, Vansh Thakur, Challen Enninful Adu, and Ram Vasudevan}
\IEEEauthorblockA{Robotics, University of Michigan, Ann Arbor, MI\\
\{cesarch, vansht, enninful, ramv\}@umich.edu}}
\maketitle
\begin{abstract}
Solving optimal control problems for high-dimensional mechanical systems subject to constraints remains challenging due to nonlinear dynamics and non-convex feasible sets.
This paper introduces \methodnamenew{}(Multiplier–Augmented Geometric Motion Algorithm), an Affine Geometric Heat Flow (AGHF)–based trajectory optimizer that incorporates an augmented Lagrangian framework solved via a primal–dual flow.
A theoretical characterization of the primal-dual AGHF is provided, establishing that \methodnamenew{} is a gradient flow on an appropriately defined energy. 
\methodnamenew{} is validated across a variety of challenging trajectory generation problems with obstacles, state limits, and actuation limits.
Comparative evaluations against state-of-the-art solvers demonstrate that \methodnamenew{} achieves high quality solutions with significantly reduced tuning effort on the hardest problems.
\end{abstract}

\section{Introduction}\label{sec:intro}
For practical deployment, an optimal control algorithm must satisfy several requirements: \textbf{(1)} computational efficiency for online or near real-time planning, \textbf{(2)} scalability to high-dimensional systems, \textbf{(3)} compatibility with nonlinear constraints, and \textbf{(4)} reliable convergence from a wide range of initializations.
Existing optimal control methods struggle to simultaneously satisfy these criteria.

Existing methods can be grouped into \emph{Dynamic Programming (DP)} and \emph{Variational methods}. 
DP-based approaches leverage Bellman’s Principle to compute optimal policies. 
Global approaches that solve the Hamilton–Jacobi–Bellman (HJB) PDE provide strong optimality guarantees, but require discretizing the entire state space, rendering them intractable beyond moderate dimensions \revision{\cite{bertsekas2012dynamic}}{\cite{bertsekas2012dynamic}}.
Local variants such as Differential Dynamic Programming (DDP)\revision{~\cite{crocoddyl2020}}{~\cite{crocoddyl2020}} improve scalability, but suffer from poor convergence when initialized far from the optimum and require careful treatment of inequality constraints.
Variational methods provide an alternative route. 
Direct methods such as collocation-based nonlinear programming \revision{\cite{C-FROST}, \cite{Tropic2020}}{\cite{C-FROST,Tropic2020}} can incorporate complex constraints, but incur high computation times and exhibit sensitivity to discretization and initialization.
Parameterized approaches such as~\cite{zhang2024rapidrobusttrajectoryoptimization} improve robustness and speed, but still struggle with convergence.

One can also perform trajectory optimization by computing the solution to the \emph{Affine Geometric Heat Flow (AGHF)} PDE\revision{~\cite{AGHF_OG}}{~\cite{AGHF_OG}}.
AGHF evolves an initial, potentially infeasible trajectory toward dynamic feasibility by minimizing control effort over a \emph{two-dimensional domain}, enabling scalability beyond HJB-based approaches.
More recent implementations \revision{~\cite{adu2024bringheatrapidtrajectory,adu2025phasingflamesrapidmotion}}{~\cite{adu2024bringheatrapidtrajectory,adu2025phasingflamesrapidmotion}} demonstrate that AGHF-based solvers can rapidly produce trajectories for high-dimensional systems.
However, these methods rely on \emph{penalty and barrier formulations} to enforce constraints, leading to \revision{three}{two} key limitations: \textbf{(i)} hyperparameter tuning to balance task costs and feasibility and \textbf{(ii)} degraded convergence when constraints compete with the objective\revision{, and \textbf{(iii)} sensitivity to the initial guess in constrained problems}{}.
A recent formulation of the AGHF \revision{\cite{kim_el_aghf} formulates}{\cite{kim_el_aghf} casts}  the problem as a primal-dual flow \revision{}{by introducing Lagrange multipliers associated with the equality constraints.}
This remedies the shortcomings of \emph{penalty and barrier functions} while enforcing constraints.
However, this formulation has only been demonstrated for kinematic constraints and while considering just a single type of cost function.

The key observation underlying our approach is that the AGHF can be identified as a gradient flow in an appropriate function space. 
This identification allows classical augmented Lagrangian primal-dual schemes to be lifted directly to the AGHF setting.
Building on this idea, this paper introduces the \emph{Multiplier-Augmented Geometric Motion Algorithm (MAGMA)}, a new AGHF-based trajectory optimization framework that incorporates an \emph{augmented Lagrangian structure solved via an augmented primal–dual geometric flow}. 
Unlike penalty-driven AGHF variants, MAGMA automatically balances constraint satisfaction and optimality by evolving primal and dual variables. 
In contrast to Kim and Park~\cite{kim_el_aghf}, MAGMA considers arbitrary nonlinear constraints and cost functions.

The primary contributions of this paper are: (i) a theoretical connection between gradient flows and the AGHF, and (ii) a formulation (MAGMA) of an augmented primal-dual flow based on the AGHF \revision{}{ for nonlinear constraints. }
All code is made available at \url{https://github.com/roahmlab/MAGMA}.

\section{Preliminaries}
\label{sec:preliminaries}
This section provides background on constrained optimization and describes the optimal control problem of interest in the remainder of this paper. 

\subsection{Constrained Optimization}

This subsection briefly formulates a constrained optimization problem that encompasses our Optimal Control Problem (OCP) of interest. 
This is done to formulate a necessary condition for optimality that our subsequent numerical implementation can leverage. 

Consider the following constrained optimization problem:
\begin{align}
    \label{eqn:constraint_opt_problem}
    \underset{x \in X}{\text{min}} \A(x) \text{ s.t. } G(x) \in K \subset X
\end{align}
where \revision{}{$X$ is a Banach space}, $ \A : X \rightarrow \R$, $G : X \rightarrow X$, $K$ a convex and closed set and  with the feasible set
\revision{$\FF := \{ x \in X : G(x) \in K \} \subset X$}{$\FF := \{ x \in X : G(x) \in K \}$, which need not be convex when $G$ is nonlinear.}

We next describe a necessary condition for optimality for the constrained optimization problem:
\begin{thm}
(Corollary 4.21 \cite{intro_nonsmooth_analysis}, Theorem 8.15 \cite{RockafellarWets1998})
\label{thm:nec_condition_cone}
Consider the constrained optimization problem in \eqref{eqn:constraint_opt_problem}, let $X$ be a Hilbert space, let $\FF \subset X$ be nonempty and closed, and let $\A: X \to \Ri$ be proper, convex and lower semi-continuous, where $\Ri = \R \cup \{-\infty,\infty\}$.
If $\bar{x}$ solves \eqref{eqn:constraint_opt_problem} then 
\begin{align}
\label{eqn:nec_cond_cone}
     0 \in \partial \A(\bar{x}) + N_{\FF}(\bar{x})
\end{align}
where $\partial \A(\bar{x})$ denotes the subdifferential of $\A$ at $\bar{x}$ and $N_\FF(\bar{x})$ denotes the normal cone to the set $\FF$ at $\bar{x}$ as defined in Lemma 4.8~\cite{intro_nonsmooth_analysis}.
\end{thm}

One could also state the necessary conditions for optimality using the KKT conditions: 
\begin{thm}(KKT conditions in Hilbert spaces, Theorem 3.27 \cite{opt_pde})
\label{thm:kkt_cond_hilbert}
Let $\bar{x}$ be a local solution of \eqref{eqn:constraint_opt_problem}, where $X$ is assumed to be a Hilbert space, at which $\A: X \to \R$ and $G: X \to X$ are continuously Fr\'echet differentiable.
Assume that Robinson's constraint qualification is satisfied at $\bar{x}$ \cite[Definition~3.16]{opt_pde}.
Then there exist a Lagrange multiplier $\bar{\lambda} \in X$ such that the Karush-Kuhn-Tucker for 
\eqref{eqn:constraint_opt_problem} in a Hilbert space
\begin{subequations}
\begin{align}
    & \mathcal{A}'(\bar{x}) + G'(\bar{x})^* \bar{\lambda} =0, \label{eqn:kkt_hilbert_gradient}\\
    & G(\bar{x}) \in K, \label{eqn:kkt_hilbert_constraint} \\
    & \bar{\lambda} \in N_K(G(\bar{x})) \label{eqn:kkt_hilbert_complementarity}
\end{align}  
\end{subequations}
are satisfied, where $\A'(\bar{x}) \in X^*$ denotes the Fre\'chet differential of $\A$ at $\bar{x}$ ~\cite[Definition 3.1]{opt_pde} and $G'(\bar{x})^*$ denotes the adjoint operator of the Fre\'chet differential of $G$ at $\bar{x}$.

\end{thm}

\subsection{The Optimal Control Problem}
The objective of this paper is to develop an algorithm to construct a trajectory beginning from some user-specified initial state $\x_0 \in \R^{2n}$, and ending in some user-specified terminal condition $\x_f$ while satisfying dynamics and inequality constraints for all $t \in [0, T]$, while minimizing a cost function.
The trajectory design problem can be posed as the following constrained optimization problem that we denote as the optimal control problem (OCP):

\begin{subequations}
\label{eqn:ocp}
\begin{align}
    \inf_{u \in L^2} \quad & \int_0^T c(\x(t), u(t)) ~dt \label{eqn:ocp_obj}\\
    \text{s.t.} \quad  &\xd(t) = \Fdx + \Fx \uu(t), \; \forall t \in [0, T], \label{eqn:ocp_dynamics_constraint} \\
    & g(\x(t), \xd(t)) \leq 0 \quad \; \forall t \in [0, T], \label{eqn:ocp_inequality_constraints}\\
    & \x(0) = \x_0,  \label{eqn:ocp_left_bc}\\
    & \x(T) = \x_f \label{eqn:ocp_right_bc}
\end{align}
\end{subequations}
where $u(t) \in \R^m$ and $L^2$ denotes the space of square integrable functions.
\eqref{eqn:ocp_dynamics_constraint} describes the dynamics constraint, $g(\x(t)) \in \R^{n_c}$ in \eqref{eqn:ocp_inequality_constraints} describes the inequality constraints, and \eqref{eqn:ocp_left_bc} and  \eqref{eqn:ocp_right_bc} describe the boundary conditions.
\revision{}{Input constraints can be incorporated according to \cite{adu2025phasingflamesrapidmotion}}.
To ensure the convergence of the AGHF PDE, we make the following assumption on the differentiability and smoothness of the system dynamics and existence of a feasible solution:
\begin{assum}\label{assum:dynamics continuity}
Both $F_d$ and $F$ are $C^2$, Lipschitz continuous, and $F$ has constant rank almost everywhere in $\R^n$.
Additionally, we assume the existence of a feasible solution to the motion planning problem for the system.
\end{assum}

\noindent Note if a particular $\x$ satisfies each of the constraints in the (OCP), then we call $\x$ a feasible trajectory of the (OCP).

\section{Gradient Flows in Hilbert Spaces}
\label{sec:gradient_flow}
This section builds a gradient flow in Hilbert spaces and shows that equilibrium points of this gradient flow satisfy the KKT-conditions in Theorem \ref{thm:kkt_cond_hilbert}.
Subsequent sections leverage this theory to formulate a variant of the AGHF PDE and to prove that it converges to points that satisfy the KKT condition. 
This section adopts the notation introduced in \eqref{eqn:constraint_opt_problem}.
Denote the evolution ``time'' of the gradient flow by $s$.
The curve at time $s$ is denoted by $x(s) \in X$ and its velocity by $x'(s) = \frac{dx(s)}{ds}$.

\subsection{Primal Flow}\label{subsec:primal-flow}

We begin by defining the Hilbert Gradient, which is used to formulate the subsequent gradient flows:

\begin{defn}[Hilbert Gradient]
    \label{defn:gradient_in_hilbert}
    Let $\A : X \to \R$ be a functional.
    Assume that $\A$ is Fr\'echet differentiable 
    denoted by $\delta \A(x) \in X^*$.
    The Hilbert Gradient of $\A$ at $x$ is the element $\nabla\A(x) \in X$ satisfying
    \begin{align}
        \delta\A(x)[h] = \langle \nabla\A(x), h \rangle_X \quad \forall\, h \in X.
    \end{align}
\end{defn}

Using this definition, we can formulate a primal flow:
\begin{defn}[Primal Flow]
\label{defn:primal_flow}
Let $\A : X \rightarrow \R$ be a Fr\'echet differentiable functional.
The primal flow associated with the constrained optimization problem in \eqref{eqn:constraint_opt_problem} is the differential inclusion
\begin{align}
    \label{eqn:primal_flow}
    -x'(s)  \in \nabla \A(x(s)) + N_\FF(x(s)), \quad x(0) \in \FF,
\end{align}
where $\nabla \A(x(s)) \in X$ denotes the Hilbert Gradient of $\A$ at $x(s)$.
\end{defn}

One can then prove that Equilibria of the primal flow satisfy a necessary condition for optimality:
\begin{lem}[Equilibria and first-order optimality]
\label{lem:primal_kkt_equilibrium}
Let $\bar{x} \in \FF$. 
Assume $\A$ is Fr\'echet differentiable at $\bar{x}$
If $\bar{x}$ is an equilibrium of the primal flow
\eqref{eqn:primal_flow}, i.e.,
\begin{equation}
    0 \in \nabla \A(\bar{x}) + N_{\FF}(\bar{x}),    
\end{equation}
then $\bar{x}$ satisfies the first-order optimality condition of
Theorem~\ref{thm:nec_condition_cone}.
\end{lem}

\begin{pf}
If $\A$ is Fr\'echet differentiable then the subdifferential has a single element $\partial \A (\bar{x}) = \{ \nabla A(\bar{x})\}$.
By applying the Riesz Representation Theorem \cite[Theorem 8.47]{Axler2020} one can identify elements of the dual with its representatives and hence one obtains \eqref{eqn:nec_cond_cone}. \qed
\end{pf}

\subsection{Augmented Primal-Dual Flow}\label{subsec:augmented-flow}
This subsection presents an augmented primal-dual gradient flow.
To simplify the exposition in this subsection, we restrict to equality constraints, which corresponds to setting $K = \{0_X\}$ in \eqref{eqn:constraint_opt_problem}, where $0_X$ denotes the origin in $X$. 
The constraint $G(x) \in K$ then reduces to $G(x) = 0_X$.
Note, that in subsequent sections we illustrate how we can consider equality or inequality constraints within (OCP). 

\begin{defn}
\label{defn:aug_primal_dual_flow}
Assume $\A, G$ are Fr\'echet differentiable. 
The augmented primal-dual flow is the differential equation:
\begin{subequations}
\begin{align}
     x'(s) &=  - ( \nabla\A(x(s)) + \nabla G(x(s))^*\lambda(s)  + \\ &\hspace*{1cm} \frac{p}{2} \nabla(||G(x)||^2) ), \label{eqn:aug_x_ode}\\
     \lambda'(s) &=  G(x(s)) \label{eqn:aug_lam_ode}
\end{align}
\end{subequations}
where $ \R \ni p\geq0$ is a user selected scalar parameter.
\end{defn}

\begin{lem}[Equilibrium points are KKT points]
\label{lem:aug_primal_dual_eq_pts}
If $(\bar{x}, \bar{\lambda}) \in X \times X $ is an equilibrium point of the dynamical system \eqref{eqn:aug_x_ode}, \eqref{eqn:aug_lam_ode} then it satisfies the KKT conditions \eqref{eqn:kkt_hilbert_gradient}-\eqref{eqn:kkt_hilbert_complementarity}.
\end{lem}

\begin{pf}
Plugging in the equilibrium into \eqref{eqn:aug_lam_ode}, we see that \eqref{eqn:eqn_dual_equilibrium_aug} corresponds to \eqref{eqn:kkt_hilbert_constraint}:
\begin{equation}
\label{eqn:eqn_dual_equilibrium_aug}
    G(\bar{x}) \in  \{0_X\} = K.
\end{equation}
Now we focus on showing \eqref{eqn:kkt_hilbert_gradient}.
Replacing the equilibrium into \eqref{eqn:aug_x_ode}, we obtain
\begin{equation}
\label{eqn:aug_kkt_grad_full}
    0 =  \nabla\A(\bar{x}) + \nabla G(\bar{x})^*\bar{\lambda} + \frac{p}{2} \nabla(\|G(\bar{x})\|^2).
\end{equation}
Computing the Hilbert Gradient of the penalty term in \eqref{eqn:aug_kkt_grad_full} we obtain
\begin{equation}
\label{eqn:penalty_vanishes}
\tfrac{p}{2}\nabla\!\left(\|G(x)\|^2\right)\bigg|_{x=\bar{x}} = p\, \nabla G(\bar{x})^* G(\bar{x}) = 0,
\end{equation}
where we have used \eqref{eqn:eqn_dual_equilibrium_aug}.
Then, replacing \eqref{eqn:penalty_vanishes} into \eqref{eqn:aug_kkt_grad_full} yields \eqref{eqn:kkt_hilbert_gradient} at the equilibrium point:
\begin{equation}
    \label{eqn:aug_kkt_grad_at_eq}
    \nabla\A(\bar{x}) + \nabla G(\bar{x})^*\bar{\lambda} = 0. 
\end{equation}
Finally, we verify \eqref{eqn:kkt_hilbert_complementarity}. 
Because $K = \{0_X\}$ is a linear subspace, its normal cone at any point is the whole space:
\begin{equation}
    N_K(G(\bar{x})) = N_{\{0_X\}}(0_X) = X.
\end{equation}
Hence $\bar{\lambda} \in X = N_K(G(\bar{x}))$ holds, establishing \eqref{eqn:kkt_hilbert_complementarity}. 
Therefore $(\bar{x}, \bar{\lambda})$ satisfies the KKT conditions \eqref{eqn:kkt_hilbert_gradient}--\eqref{eqn:kkt_hilbert_complementarity}. \qed
\end{pf}

\section{The Affine Geometric Heat Flow (AGHF) Partial Differential Equation}
\label{sec:AGHF}

This section summarizes the background knowledge and theory of the AGHF.
A more detailed treatment of the subjects discussed here can be found in \cite{AGHF_OG,liu2019homotopy,liu2020geometric}.

\subsection{Homotopies and the AGHF}

Consider a homotopy: $\X:[0,T] \times [0,s_{max}] \to \R^{2N}$, that is twice differentiable with respect to its first argument and differentiable with respect to its second argument. 
Define the \emph{Action Functional}:
\begin{equation}
\label{eqn:action functional}
    \Act(\X_s) = \int_0^T  L(\X_s(t),\dot{\X}_s(t)) dt.
\end{equation}

Using these definitions, one can define the AGHF PDE:
\begin{defn}\label{defn:AGHF}
The \emph{Affine Geometric Heat Flow} is the parabolic Partial Differential Equation:
\begin{align}
\label{eqn:AGHF}
\dxds(t,s) = \Minv(\X(t,s)) \bigg(&\ddt \dldxds(\X_s(t),\dot{\X}_s(t)) \notag \\
&- \dldxs(\X_s(t),\dot{\X}_s(t))\bigg)
\end{align}
with the following boundary conditions:
\begin{align}
\X_s(0) &= \x_0,  \quad \forall s \in [0,s_{max}] \\ 
\X_s(T) &= \x_f, \quad \forall s \in [0,s_{max}]
\end{align}
and where $M(\X(t, s)) \succeq 0$.
\end{defn}

When solving the AGHF PDE, one begins by specifying an initial curve $\X_{init}:[0,T] \to \R^{2N}$ and setting it such that $\X_0 = \X_{init}$.
As the AGHF PDE evolves forward in $s$, one can prove that the action functional is monotonically decreasing \revision{}{\cite{adu2024bringheatrapidtrajectory, aghf_arxiv_og}}. 
In addition, if during that evolution the AGHF converges to a curve where the right hand side of the AGHF PDE is equal to $0$, then one has found a curve that extremizes the action functional.
Such a curve is called a \emph{steady state solution}.

The AGHF stated in Definition \ref{defn:AGHF} is in fact a gradient flow:
\begin{lem}[AGHF - Primal Gradient Flow]
\label{lem:aghf_gf_rel}
The AGHF \eqref{eqn:AGHF} is the negative Hilbert Gradient (Definition~\ref{defn:gradient_in_hilbert}) of the action functional \eqref{eqn:action functional}:
\begin{equation}
\label{eqn:aghf_as_gf}
    \frac{\partial \X}{\partial s}(t,s) = -\nabla \Act(\X_s)(t),
\end{equation}
which is an unconstrained primal-dual flow with $\lambda \equiv 0$ as in Definition \ref{defn:aug_primal_dual_flow}.
\end{lem}

\begin{pf}
Fix $s$ and suppose $\X_s(0), \X_s(T)$ are held fixed. 
Compute the Fr\'echet Derivative of $\Act$ at $\X_s$ in a direction $h \in X$ satisfying $h(0) = h(T) = 0$. 
Integration by parts yields the Euler--Lagrange expression
\begin{equation}
\begin{split}
    \delta\Act(\X_s)[h] = \int_0^T \Bigl( &\dldxs(\X_s(t),\dot{\X}_s(t)) \\
    &- \ddt \dldxds(\X_s(t),\dot{\X}_s(t)) \Bigr)^T h(t)\, dt
\end{split}
\end{equation}
Introducing the weighted inner product $\langle a, b \rangle_M := \int_0^T a(t)^T \allowbreak M(\X(t,s))\, b(t)\, dt$, we identify the Hilbert gradient (Definition~\ref{defn:gradient_in_hilbert}) as the unique element $\nabla\Act(\X_s) \in X$ satisfying $\delta\Act(\X_s)[h] = \langle \nabla\Act(\X_s),\, h \rangle_M$ for all admissible $h$. 
Comparing the two expressions gives the desired result. \qed
\end{pf}

\subsection{Ensuring the AGHF coincides with the (OCP)}
\label{subsec:aghf_ocp}

To ensure that the action functional being minimized coincides with the (OCP) in \eqref{eqn:ocp}, we must design $L$ carefully.
First, define the \emph{Lagrangian} as:
\begin{align}
\label{eqn:lagrangian}
L(\X_s(t),\dot{\X}_s(t))
&= \left(\dot{\X}_s(t) - F_d(\X_s(t))\right)^T
   J(\X_s(t)) \notag \\
&\quad \cdot \left(\dot{\X}_s(t) - F_d(\X_s(t))\right)
\end{align}
where $J: \R^{2N} \to \R^{2N \times 2N}$ is a user-specified matrix. 

Now, if we assume $c(x(t), u(t)) = \|u(t)\|^2$ in \eqref{eqn:ocp}, then we can design $J$ by applying \cite[Theorem 1]{AGHF_OG}.
For a general cost function one can follow \cite[Definition 6]{adu2025phasingflamesrapidmotion}.

Under Assumption~\ref{assum:dynamics continuity}, \cite[Theorem~1]{AGHF_OG} specifies a control extraction procedure and establishes that under mild conditions the resulting trajectory is arbitrarily close to a feasible trajectory of (OCP). 
In fact, \cite[Theorem~1]{AGHF_OG} provides an explicit bound quantifying this closeness. 
We omit these details here for space reasons.

\section{The Augmented Lagrangian Affine Geometric Heat Flow}
\label{sec:alm_aghf}
This section extends the AGHF equation presented in Definition \ref{defn:AGHF} to handle constraints.
Note that Lemma \ref{lem:aghf_gf_rel} allows us to interpret the AGHF as a Gradient Flow.
We use Definition \ref{defn:aug_primal_dual_flow} to construct an augmented primal-dual AGHF.
Building upon the AGHF framework established in Section \ref{sec:AGHF}, define a modified Lagrangian that incorporates constraint information:
\begin{align}
\label{eqn:auglagrangian}
L'(\X_s(t), \dot{\X}_s(t), \mu_s(t))
&= L(\X_s(t), \dot{\X}_s(t))
  \notag \\ + \mu_s(t)^T W(g(\X_s(t))) 
&\quad + \frac{p}{2} \|W(g(\X_s(t)))\|_2^2,
\end{align}
where $\R \ni p > 0$ represents the scalar penalty parameter, $g(\X_s(t)) \in \R^{n_c}$ with $n_c$ the number of constraints and the function $W$ is applied element-wise.
In the case of equality constraints, $W$ is the identity and in the case of inequality constraints let $W(g(\X_s(t))) = \max(0, g(\X_s(t)))$.
Let $W^{\ccons}$ denote the smooth approximation of $W$ for inequality constraints.

By adapting Definition \ref{defn:aug_primal_dual_flow}, one can build the following:
\begin{defn}(MAGMA)
The augmented primal-dual AGHF is a partial differential equation that is defined as:
\begin{subequations}
\begin{align}
\dxds(t,s) &= \Minv(\X(t,s)) \bigg(
   \ddt \dlpdxds(\X_s(t),\dot{\X}_s(t), \mu_s(t)) \notag \\
   &\quad - \dlpdxs(\X_s(t),\dot{\X}_s(t), \mu_s(t))
   \bigg) \label{eqn:magma_primal} \\
\frac{\partial \mu_s(t)}{\partial s}
   &= \frac{\partial L'}{\partial \mu_s}
      (\X_s(t),\dot{\X}_s(t), \mu_s(t)) \label{eqn:magma_dual}
\end{align}
with the following boundary conditions:
\begin{align}
\X_s(0) &= \x_0,  \quad \forall s \in [0,s_{max}] \\ 
\X_s(T) &= \x_f, \quad \forall s \in [0,s_{max}]
\end{align}
and where $M(\X(t, s)) \succeq 0$.
\end{subequations}
\end{defn}

\section{Experiments and Results}
\label{sec:experiments}

\begin{table*}[t]
\centering
\caption{Per-scenario objective cost and solve time (seconds) across 6 sphere obstacle scenarios and 8 hard task-based scenarios.
\textbf{F} denotes failure to satisfy the success criteria of Section~\ref{sec:experiments} within the time budget (60\,s for sphere, 120\,s for hard scenarios).
PHLAME failed on all 14 scenarios and is omitted.
Success rate is the fraction of scenarios solved successfully.}
\label{tab:combined_results}
\footnotesize
\setlength{\tabcolsep}{3.5pt}
\renewcommand{\arraystretch}{1.1}
\begin{adjustbox}{max width=\textwidth}
\begin{tabular}{ll cccccc c cccccccc c c}
\toprule
 & & \multicolumn{6}{c}{Sphere scenarios} & & \multicolumn{8}{c}{Hard scenarios} & & \\
\cmidrule(lr){3-8} \cmidrule(lr){10-17}
Method & Metric & 1 & 2 & 3 & 4 & 5 & 6 & & 1 & 2 & 3 & 4 & 5 & 6 & 7 & 8 & & Success \\
\midrule
\multirow{2}{*}{\textcolor{algreen}{MAGMA}}
 & Cost & 819.7 & 723.9 & 606.8 & 970.1 & 639.8 & 802.2 & & 423.0 & 206.2 & 293.3 & 974.42 & 341.74 & 1085.3 & 830.4 & 565.7 & & \multirow{2}{*}{14/14} \\
 & Time &   8.9 &   6.6 &   5.4 &   7.7 &   8.0 &  10.6 & &  22.3 &  17.8 &  20.5 &  17.5 &  23.4 &   17.2 &  11.9 &  22.1 & & \\
\midrule
\multirow{2}{*}{\textcolor{blazepurple}{BLAZE}}
 & Cost & \textbf{F} & 559.9 & 806.9 & 412.8 & 743.8 & 397.2 & & 456.3 & 168.9 & \textbf{F} & \textbf{F} & 523.9 & 872.2 & 691.5 & \textbf{F} & & \multirow{2}{*}{10/14} \\
 & Time & \textbf{F} &   6.0 &   6.7 &   8.5 &   7.9 &   6.6 & &   9.2 &   6.6 & \textbf{F} & \textbf{F} &   6.8 &   6.9 &   8.0 & \textbf{F} & & \\
\midrule
\multirow{2}{*}{\textcolor{aligatororange}{Aligator}}
 & Cost & 589.4 & 536.9 & 720.7 & 542.2 & \textbf{F} & 483.1 & & \textbf{F} & 187.7 & \textbf{F} & \textbf{F} & \textbf{F} & \textbf{F} & \textbf{F} & \textbf{F} & & \multirow{2}{*}{6/14} \\
 & Time &  32.1 &  19.2 &  41.6 &  30.0 & \textbf{F} &  21.7 & & \textbf{F} &  72.2 & \textbf{F} & \textbf{F} & \textbf{F} & \textbf{F} & \textbf{F} & \textbf{F} & & \\
\bottomrule
\end{tabular}
\end{adjustbox}
\end{table*}

This section evaluates the speed and performance of \methodnamenew{} on a variety of constrained optimization problems.
All code describing the implementation from this section is available at \url{https://github.com/roahmlab/MAGMA}.

\begin{figure*}[t]
    \centering
    \subfloat{\begin{minipage}[t]{0.24\textwidth}
        \centering
        \adjincludegraphics[width=\textwidth, trim={.15\width .15\height .15\width .15\height}, clip]{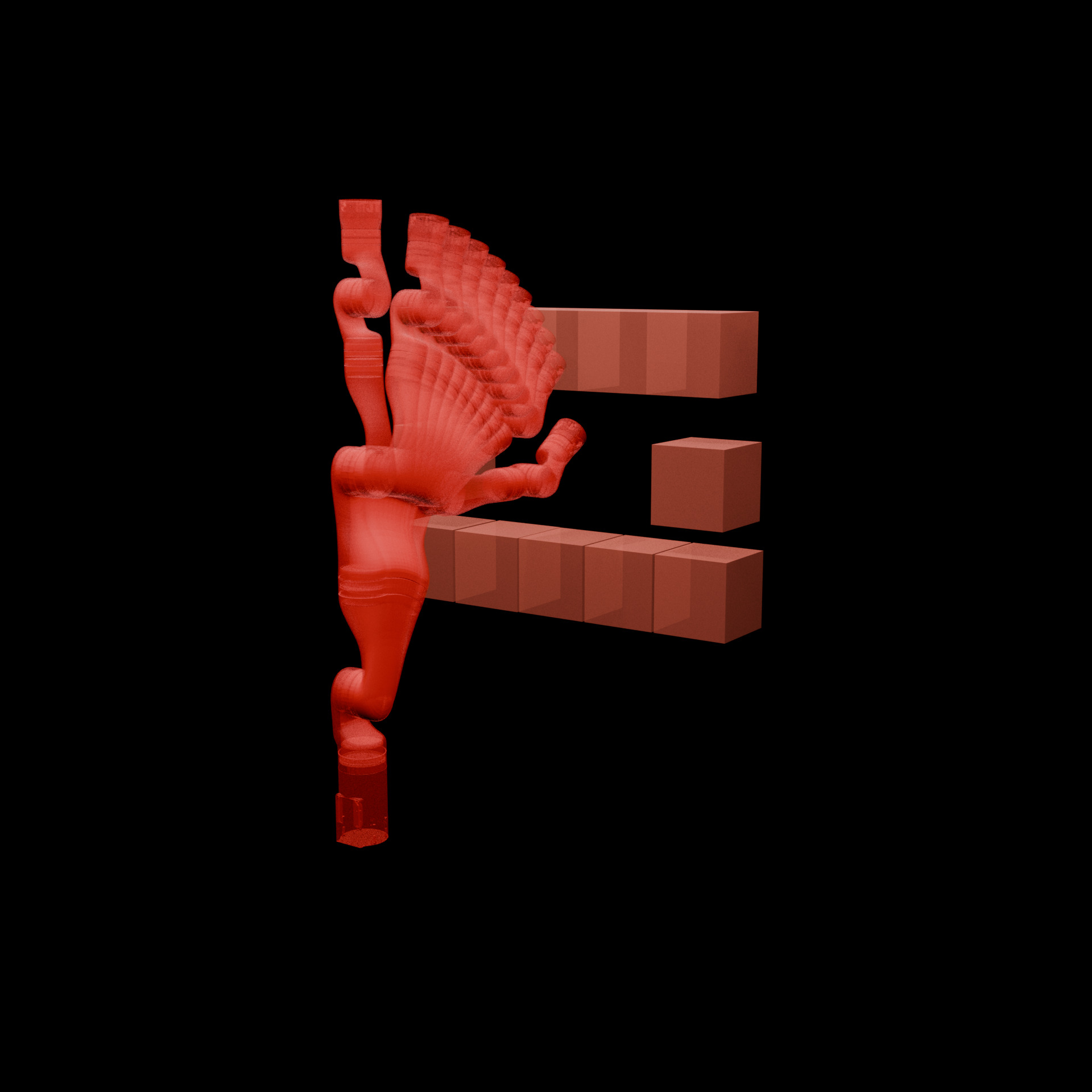}
    \end{minipage}}
    \hfill
    \subfloat{\begin{minipage}[t]{0.24\textwidth}
        \centering
        \adjincludegraphics[width=\textwidth, trim={.15\width .15\height .15\width .15\height}, clip]{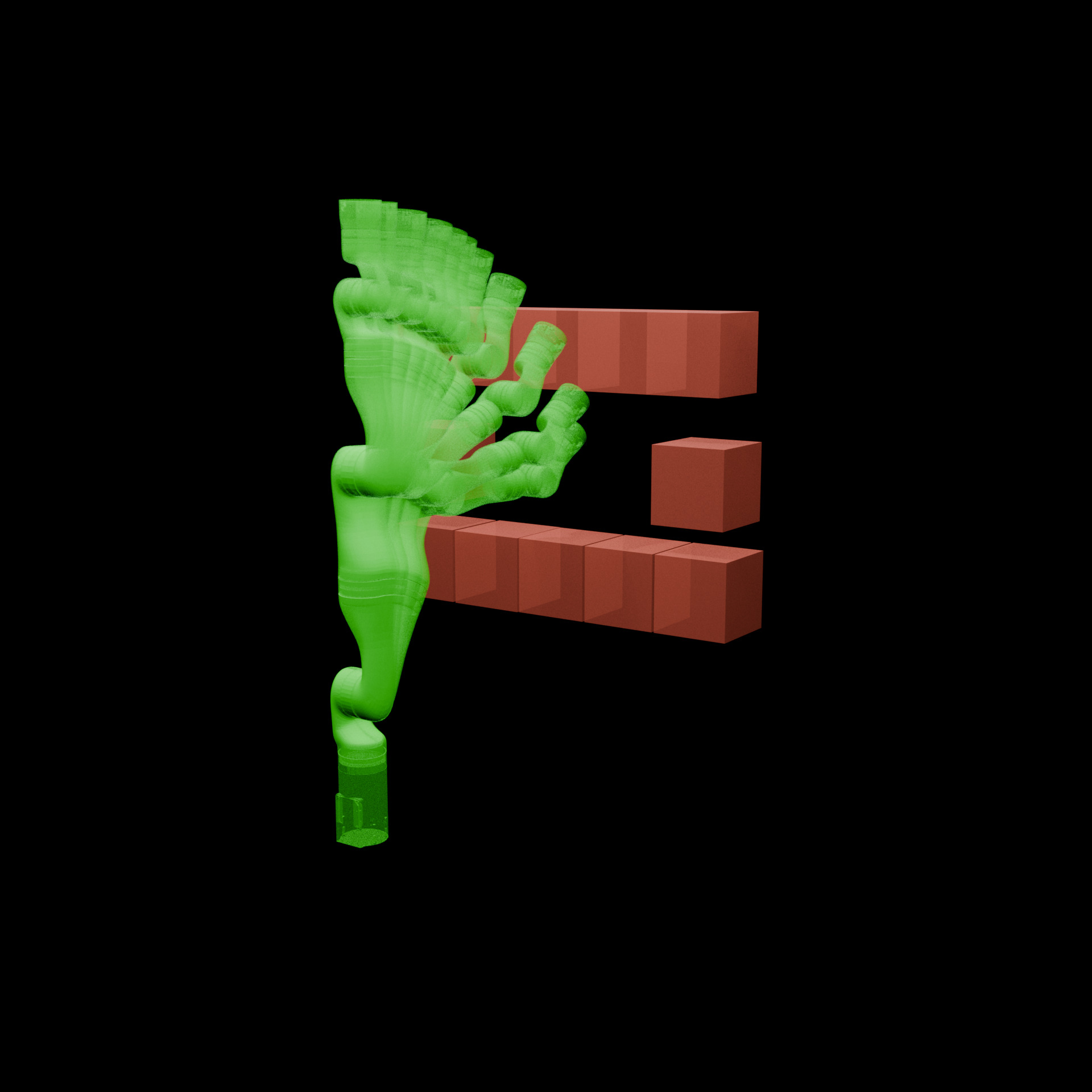}
    \end{minipage}}
    \hfill
    \subfloat{\begin{minipage}[t]{0.24\textwidth}
        \centering
        \adjincludegraphics[width=\textwidth, trim={.15\width .15\height .15\width .15\height}, clip]{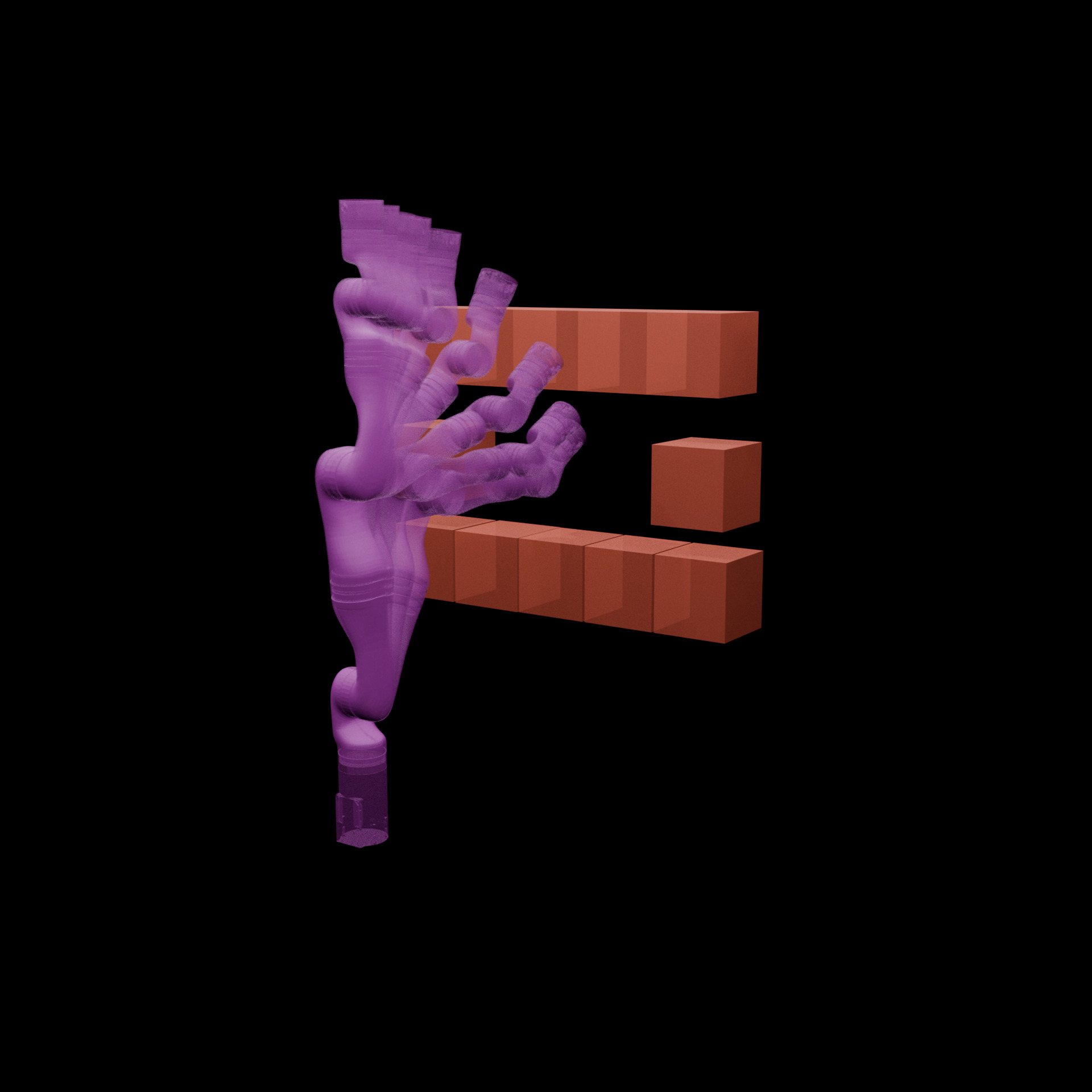}
    \end{minipage}}
    \hfill
    \subfloat{\begin{minipage}[t]{0.24\textwidth}
        \centering
        \adjincludegraphics[width=\textwidth, trim={.15\width .15\height .15\width .15\height}, clip]{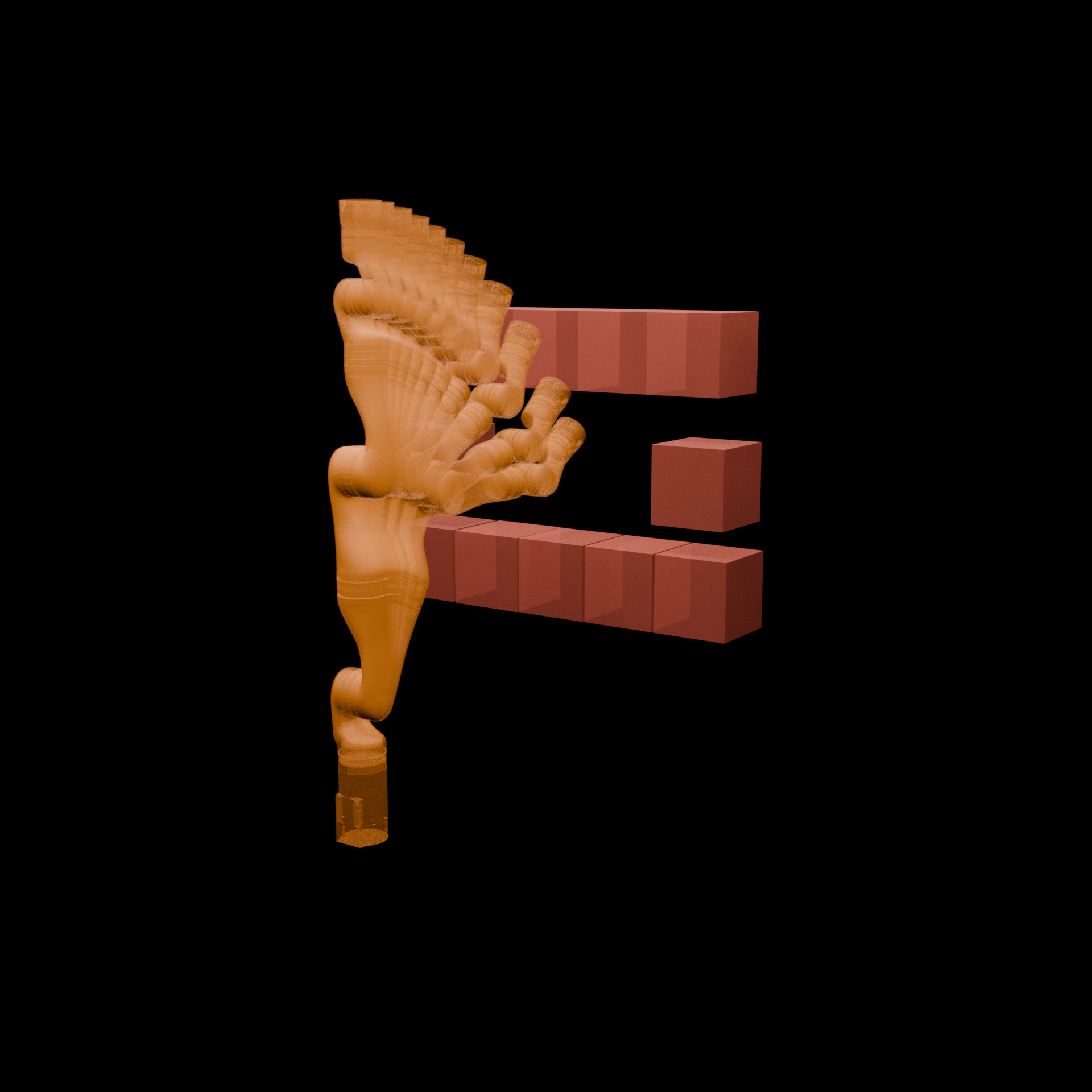}
    \end{minipage}}

    \vspace{2pt}

    \subfloat[\textcolor{red_initial_guess}{\textbf{Initial Guess}}]{\begin{minipage}[t]{0.24\textwidth}
        \centering
        \includegraphics[width=\textwidth]{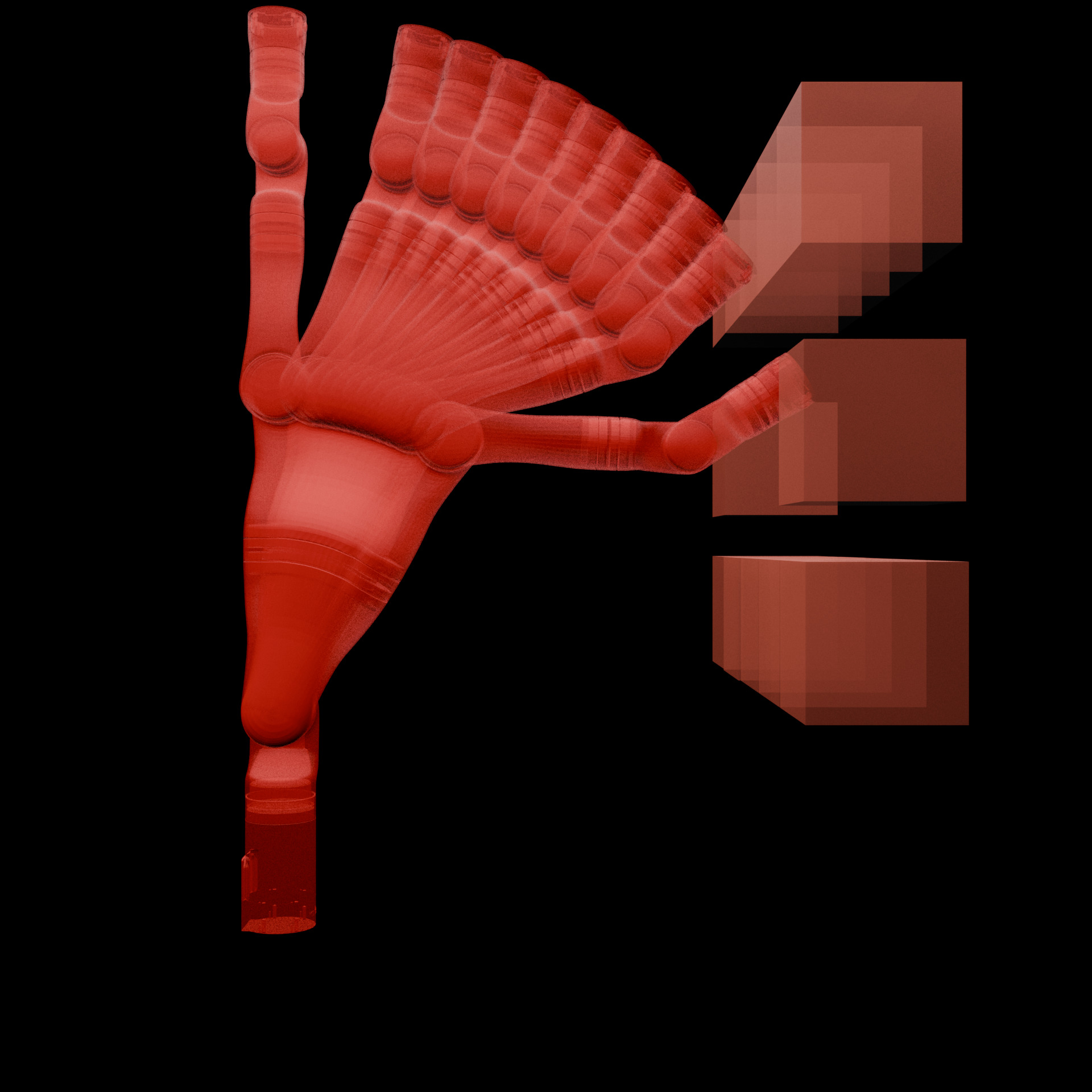}
        
        \label{fig:init3}
    \end{minipage}}
    \hfill
    \subfloat[\textcolor{algreen}{\textbf{MAGMA}}]{\begin{minipage}[t]{0.24\textwidth}
        \centering
        \includegraphics[width=\textwidth]{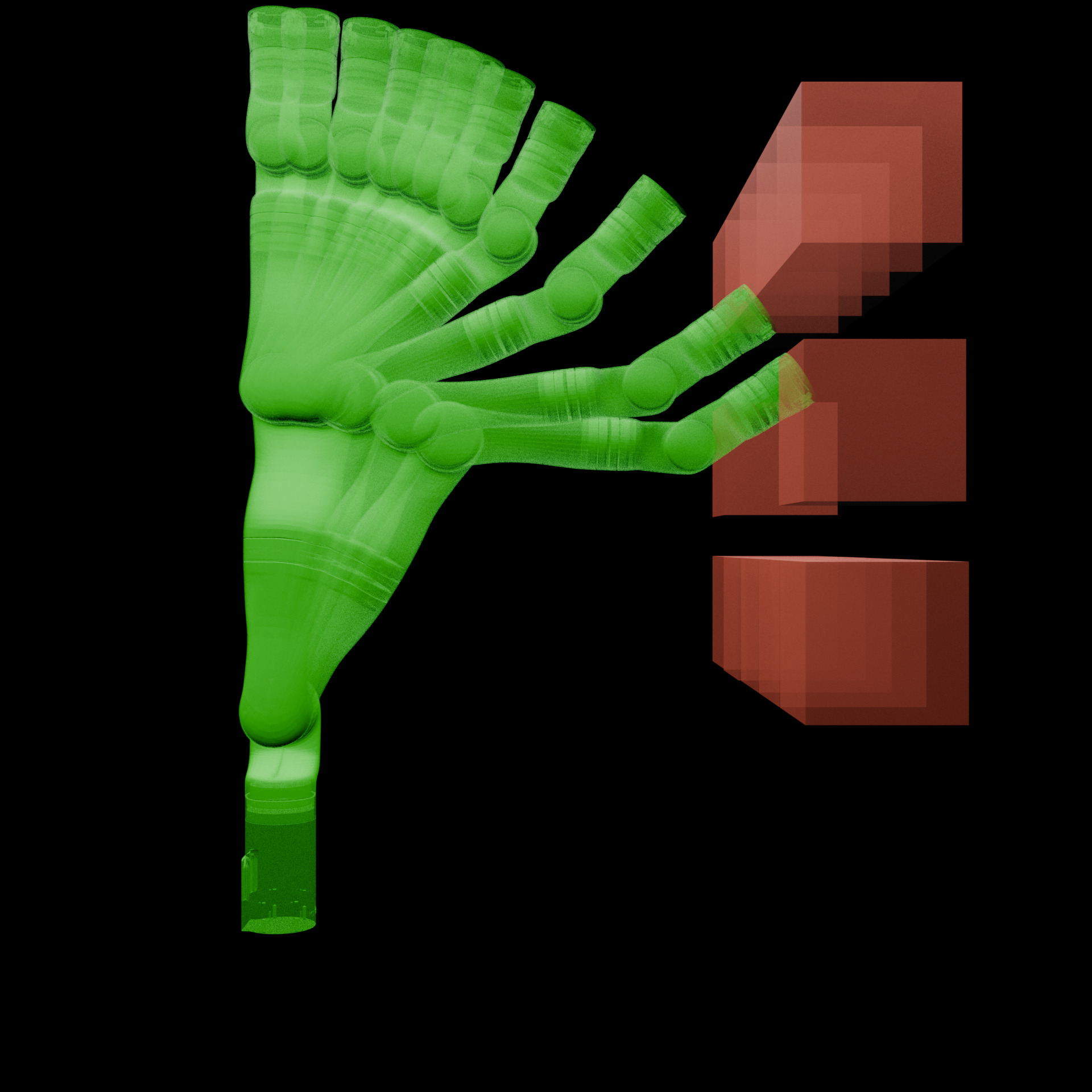}
        
        \label{fig:mg3}
    \end{minipage}}
    \hfill
    \subfloat[\textcolor{blazepurple}{\textbf{BLAZE}}]{\begin{minipage}[t]{0.24\textwidth}
        \centering
        \includegraphics[width=\textwidth]{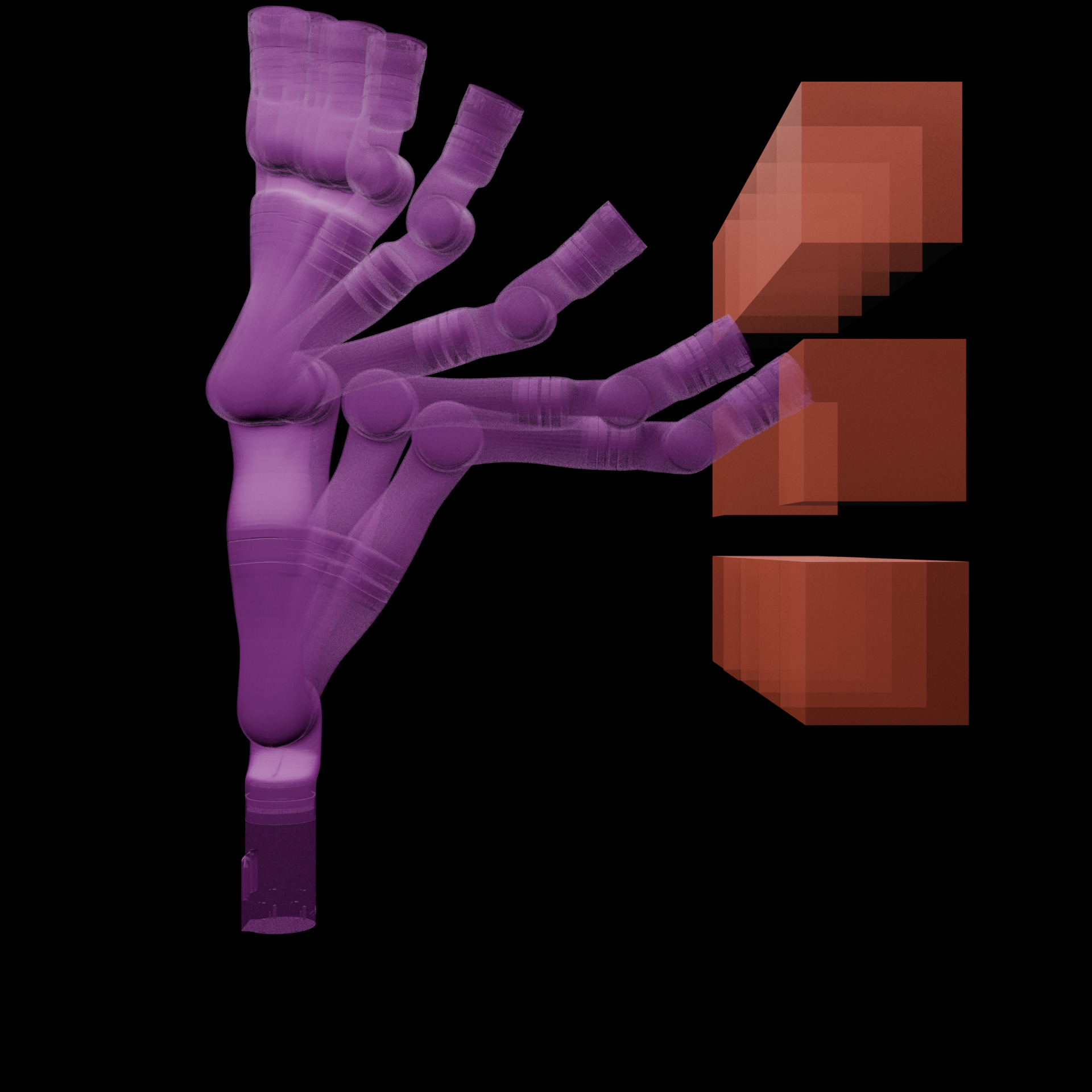}
        
        \label{fig:blaze3}
    \end{minipage}}
    \hfill
    \subfloat[\textcolor{aligatororange}{\textbf{Aligator}}]{\begin{minipage}[t]{0.24\textwidth}
        \centering
        \includegraphics[width=\textwidth]{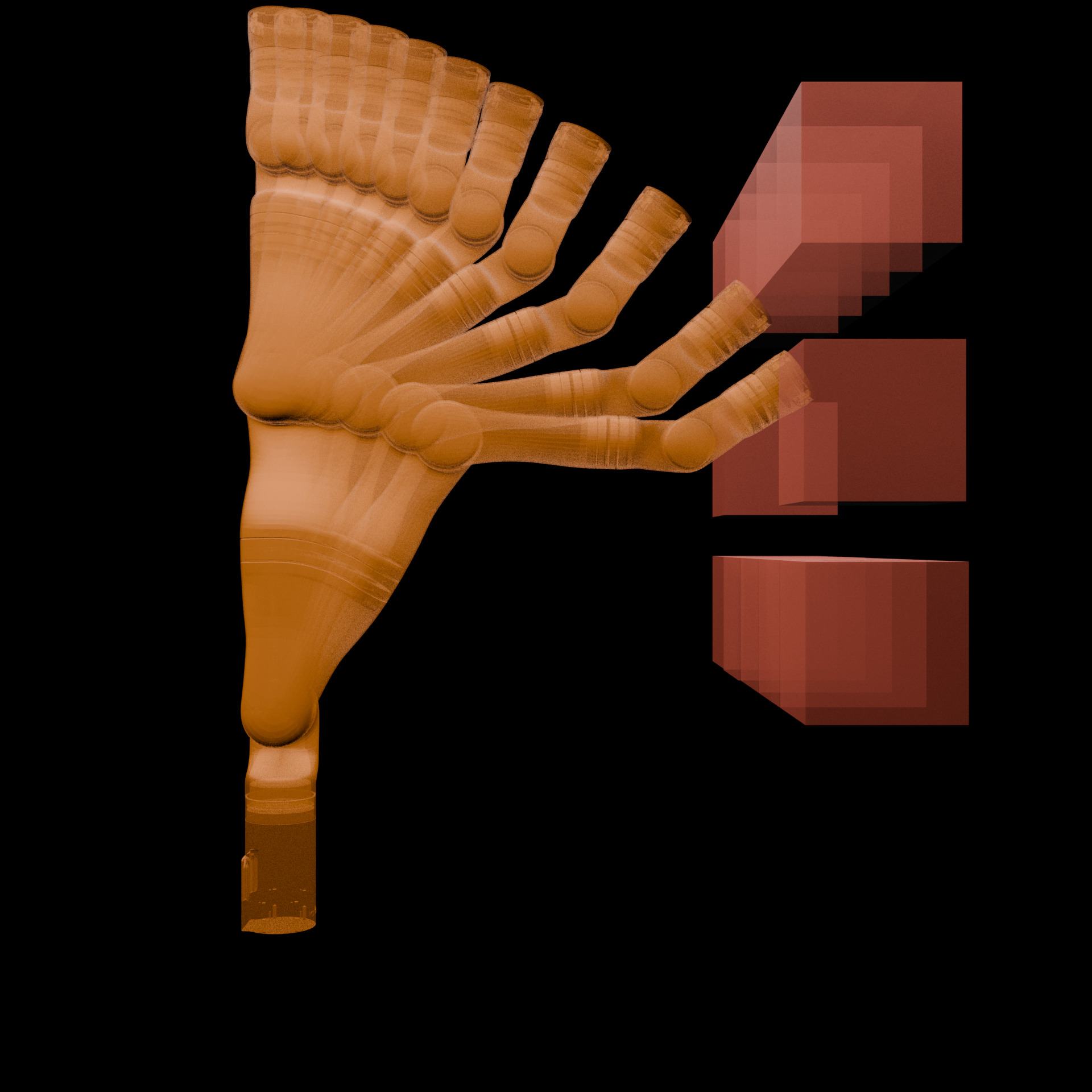}
        
        \label{fig:alig3}
    \end{minipage}}

\caption{Qualitative comparison of the solution trajectories on Hard Scenario~2 (see Table~\ref{tab:combined_results}).
The \textcolor{red_initial_guess}{Initial Guess} (left column) is a straight line trajectory in configuration space that collides with the obstacles.
\textcolor{algreen}{MAGMA} pushes this initial guess out of collision and converges to a feasible solution in 17.8\,s, \textcolor{blazepurple}{BLAZE} does it in 6.6\,s and \textcolor{aligatororange}{Aligator} in 72.2\,s.
The top row shows an isometric view; the bottom row shows a side view of the same trajectories, chosen to make the interaction with the obstacle visible.}
\label{fig:hard_task_scenario_2}
\end{figure*}

\subsection{Implementation of \methodnamenew{}}

The numerical implementation of \methodnamenew{} follows the pseudo-spectral method-of-lines approach adopted in \cite{adu2024bringheatrapidtrajectory} which is summarized in Section \ref{sec:alm_aghf}.
To evaluate the derivatives of $W$ in \eqref{eqn:magma_primal} and \eqref{eqn:magma_dual}, we approximate it by a smooth function $W(y) \approx W^{\ccons}(y) = y \cdot S(y)$ where the activation function is defined by:
\begin{equation}
\label{eq: barrier lagrangian}
    S(y) = \frac{1}{2} + \frac{1}{2} \tanh(\ccons \cdot y)
\end{equation}
where $\ccons$ is a hyper-parameter that determines how fast $S(y)$ transitions from 0 to 1.

\subsection{Experimental Setup}
We compare \methodnamenew{} to the DDP method Aligator \cite{aligator} and the AGHF-based methods BLAZE and \methodname{} \cite{adu2025phasingflamesrapidmotion} which use a penalty based formulation to handle constraints.
In all the problems, we enforce input constraints and state constraints and in problems with obstacles we also enforce obstacle avoidance constraints.
All the experiments were run on an Ubuntu 22.04 machine with an AMD EPYC 7742 64-Core @ 256x 2.25GHz CPU.

\subsubsection{Selection of Parameters}

For the BLAZE and PHLAME baselines, we adopt the parameter configurations reported in \cite{adu2025phasingflamesrapidmotion}.
In contrast, \methodnamenew{} does not require extensive hyperparameter tuning.
A single reasonable choice of the penalty coefficients was sufficient to obtain all reported results; no grid search or parameter sweep was performed.
To ensure fair comparison across methods, all experiments use a common initial guess $\Xinit$ as was described in \cite{adu2025phasingflamesrapidmotion}.

\subsubsection{Evaluating Success or Failure}
\label{subsec: success}
Each numerical method under evaluation produces an open-loop control signal $\uu^*:[0,T] \to \R^m$ along with the corresponding state trajectory $\X^*:[0,T] \to \R^{2N}$.
A key question is whether the robot can faithfully execute these computed solutions during physical operation.
To address this and properly account for integration errors, we evaluate constraint satisfaction by simulating forward dynamics under a feedback controller:
\begin{equation}
    \Ufb(t) = \uu^*(t) +  k_p (\X^*_{P1}(t) - \q(t)) + k_v (\X^*_{P2}(t) - \qd(t)),
\label{eq: feed forward control}
\end{equation}
with gains $k_p = k_v = 100$ applied uniformly across all experiments, where $\q(t)$ and $\qd(t)$ denote the instantaneous robot position and velocity.
We propagate the system dynamics using \eqref{eq: feed forward control} and then assess both tracking performance and constraint compliance.
\revision{}{We adopt this approach mirroring what is done in real-world deployments.}

A trial qualifies as \emph{successful} in obstacle-free scenarios when two conditions hold: the infinity norm between the forward-integrated terminal state and $\X_f$ falls below $\epsilon = 0.05$, and all state values throughout the trajectory stay within ±0.5 of their specified bounds.
For scenarios involving obstacles, success additionally requires that joint positions sampled at $\Delta t = 10^{-2}$s intervals avoid collision with all obstacle regions at every joint frame, assuming buffered obstacles by 0.1\,cm.
Note that when solving the obstacle constraints, all methods are implemented with a 0.1 cm buffer (obstacles are grown by 0.1 cm) to provide a small safety margin during trajectory execution.

\subsection{Kinova Gen3 Sphere Obstacle Avoidance}

This section compares \methodnamenew{}, BLAZE \cite{adu2025phasingflamesrapidmotion}, PHLAME \cite{adu2024bringheatrapidtrajectory} and  Aligator \cite{aligator} in solving the (OCP) in multiple scenarios with sphere obstacles.

In all these experiments we design $\Xinit$ to be in collision.
This allows us to investigate how effective the proposed method is in handling unfeasible initial guesses.

We evaluate each method on 6 randomly generated scenarios, each containing either 5 or 10 obstacles. 
In all cases, the obstacles are placed randomly and the initial guess $\Xinit$ is constructed to be in collision.
We gave a 60s time budget to be able to generate a trajectory.
Being unable to generate the trajectory within that time or not satisfying the earlier success criteria from Section \ref{subsec: success} is counted as an unsuccessful trial.
 Table~\ref{tab:combined_results} summarizes the results.
 Videos illustrating the performance of the various algorithms can be found at \url{https://roahmlab.github.io/MAGMASite/}.
 \methodnamenew{} achieves 100\% success rate.

\subsection{Kinova Gen3 Hard Task-Based Scenarios}

This section compares \methodnamenew{}, {\methodname} \cite{adu2024bringheatrapidtrajectory}, Aligator and BLAZE under a similar setup as the previous subsection, with two key differences. 
First, obstacles are now cuboids arranged to resemble more task-oriented scenarios for manipulators that would be deployed in domestic and industrial settings.
These scenarios involve actions such as reaching into shelves, retracting arms out of shelves and under shelves, and placing items in confined bins, reflecting real-world manipulation tasks.
Figure \ref{fig:hard_task_scenario_2} depicts Scenario 2 comparing the solution of the different methods and  Figure \ref{fig:scenario4} depicts the solution of \methodnamenew{} for Scenario 7.
Second, the time limit for each method is extended to 120s to accommodate scenarios with a greater number of obstacles.
For these experiments each method must generate a 2-second long trajectory, which is a relatively short time to execute many of these scenarios due to how far the robot must move.
This time constraint forces the methods to produce dynamic solutions while simultaneously satisfying state and input constraints, as well as obstacle avoidance, making the scenarios challenging.
A trial is considered unsuccessful if a method fails to generate a trajectory within the allotted time or does not satisfy the earlier success criteria from Section \ref{subsec: success}.

Table \ref{tab:combined_results} summarizes the performance of the three optimal control algorithms for the 8  different realistic scenarios.
\methodnamenew{} demonstrates a higher success rate, completing 100\% of trials.

\begin{figure}[h]
\subfloat{\begin{minipage}{0.48\linewidth}

        \includegraphics[width=1\columnwidth]{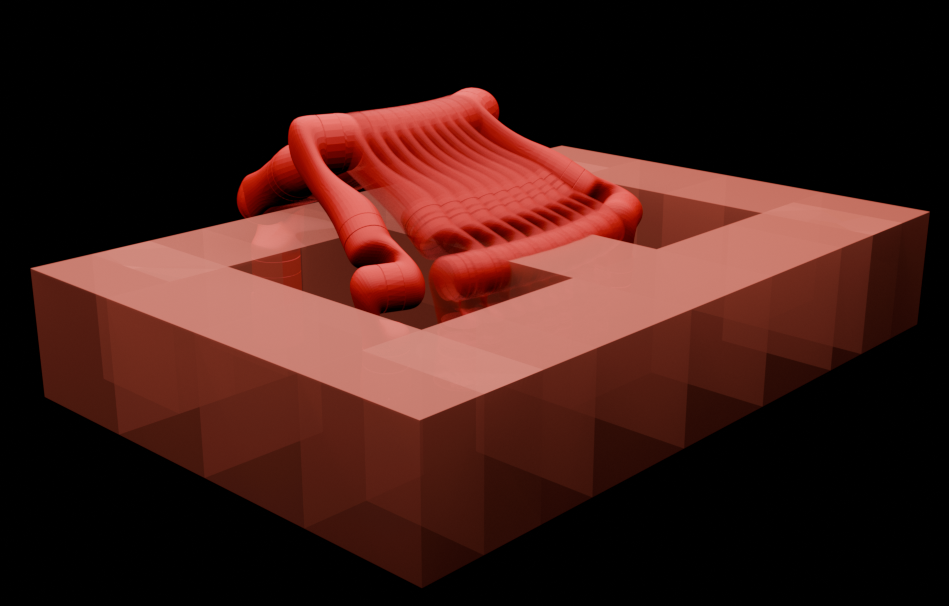}
\end{minipage}}
\hfill
\subfloat{\begin{minipage}{0.48\linewidth}
        \includegraphics[width=1\columnwidth]{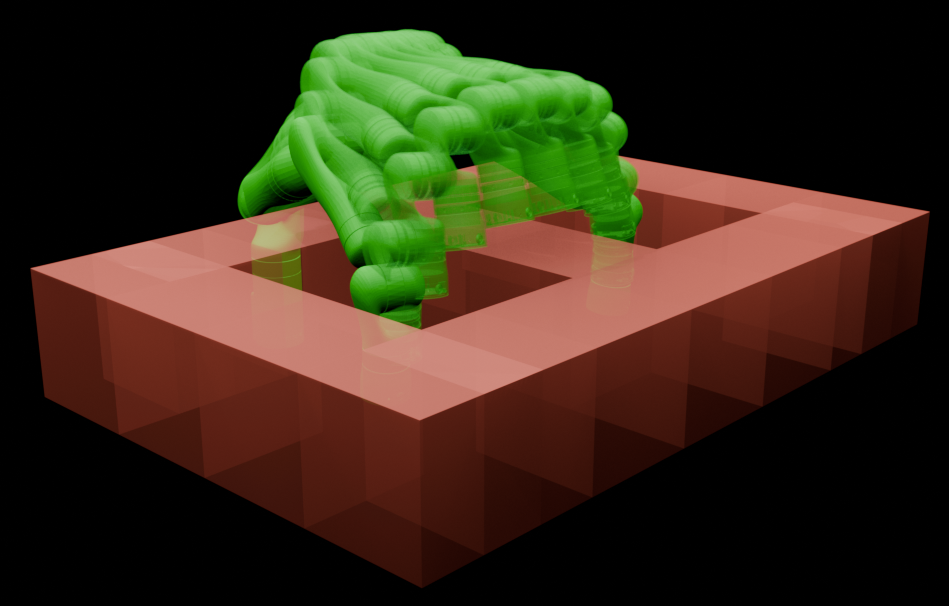}
\end{minipage}}
\caption{This figure shows a visualization of one of the task-based scenarios (Scenario 7). 
The \textcolor{red_initial_guess}{initial trajectory} for the scenario is shown in the left image in red and collides with the obstacles along the path. 
\textcolor{algreen}{MAGMA} is able to push this initial guess out of collision and generate the optimal collision-free solution shown in green in 11.9s for this scenario.}
\label{fig:scenario4}
\end{figure}

\section{Conclusion}
\label{sec:conclusion}
This paper introduces \methodnamenew{}, a trajectory optimization framework that combines an augmented Lagrangian formulation with the Affine Geometric Heat Flow to address constrained optimal control problems. 
By leveraging an augmented primal–dual gradient flow, the proposed method reduces reliance on manual penalty tuning while providing a systematic mechanism for balancing feasibility and optimality. 
To do this, the paper establishes a connection between gradient flows and KKT conditions in Hilbert Spaces, and demonstrates through experiments on high-dimensional robotic systems that the method achieves competitive robustness, accuracy, and computational performance relative to existing approaches.

\bibliographystyle{IEEEtran}
\bibliography{references}


\end{document}